\documentclass[journal]{IEEEtran}
\usepackage{amsmath,bm,amsfonts,amssymb, amsthm}
\usepackage{graphicx}
\usepackage{subcaption}
 \usepackage{url}
\usepackage{algorithm}
\usepackage[font=small,labelfont=bf]{caption}
\usepackage{hyperref}
\usepackage{color,xcolor,ucs}
\usepackage{comment}
\usepackage[noend]{algpseudocode}%
\usepackage{fancyhdr}
\newcommand{\B}{\mathbf}
\newcommand{\C}{\mathcal}
\newcommand{\R}{\mathrm}

\newcommand{\RR}[1]{{\color{black} #1}} 
\newcommand{\TYPO}[1]{{\color{black} #1}} 
\newcommand{\SHORT}[1]{{\color{black} #1}} 

\newcommand{\MR}[1]{{\color{black} #1}} 

\newtheorem{lemma}{Lemma}

\theoremstyle{definition}
\newtheorem{definition}{Definition}

\definecolor{mygray}{gray}{0.6}

\makeatletter

\def\ps@headings{%
    \def\@oddhead{%
        \ifodd\value{page}%
            {\scriptsize RATH ET AL.: TOWARDS KINEMATIC ACTIONABLE INFEASIBILITY DETECTION IN MOTION PLANNING}\hfil
            {\scriptsize\thepage}%
        \else
            {\scriptsize\thepage}\hfil
            {\scriptsize IEEE ROBOTICS AND AUTOMATION LETTERS. PREPRINT VERSION. ACCEPTED SEP, 2026}%
        \fi
    }%
    \let\@evenhead\@oddhead
    \def\@oddfoot{}%
    \def\@evenfoot{}%
}

\def\ps@IEEEtitlepagestyle{%
    \def\@oddhead{%
        {\scriptsize IEEE ROBOTICS AND AUTOMATION LETTERS. PREPRINT VERSION. ACCEPTED SEP, 2026}\hfil
        {\scriptsize\thepage}%
    }%
    \let\@evenhead\@oddhead
    \def\@oddfoot{}%
    \def\@evenfoot{}%
}

\makeatother

\begin{document}

\title{\huge Towards Kinematic Actionable Infeasibility Detection\\in Motion Planning}
\vspace{-5mm}
\author{Aayush Rath, Lakshya Jindal, and Antony Thomas
\thanks{Manuscript received: May 10, 2026; Revised: August 1, 2026; Ac-
cepted: September 25, 2026. This paper was recommended for publication by Editor Lucia Pallottino upon evaluation of the Associate Editor and Reviewers’
comments. This work was supported by the International Institute of Information Technology Hyderabad (IIIT Hyderabad) under Grant No. IIIT/R\&D Office/Seed-Grant/2025-26/002.\\
The authors are with the Robotics Research Center, IIIT Hyderabad, Hyderabad 500032, India. email: {\tt\footnotesize anusandhaan1@gmail.com, antony.thomas@iiit.ac.in}}
}

\IEEEaftertitletext{\vspace{-2.5em}} 

\maketitle
\begin{abstract}
Motion planning in robotics requires not only computing collision-free paths but also certifying infeasibility when no such path exists. Complete methods are limited to low-dimensional spaces, while sampling-based planners scale efficiently but cannot provide finite-time infeasibility certificates, leaving this problem largely unresolved in high-dimensional spaces. In this letter, we present a geometry-driven framework for certifying infeasibility through an explicit \MR{resolution-dependent} analysis of configuration space topology. Leveraging signed distance field representations, the proposed method traces separating manifolds induced by obstacle boundaries directly in configuration space, enabling both detection of infeasibility and identification of the specific geometric cause. To address computational challenges, we develop a parallel frontier-expansion algorithm that exploits GPU acceleration for efficient simplicial reconstruction in high-dimensional spaces. We validate the approach on 4-DOF and 5-DOF robot scenarios, certifying infeasibility within seconds for 4-DOF cases and under four minutes for 5-DOF cases. We further discuss avenues for improving scalability to higher-dimensional spaces.
\end{abstract}
\IEEEpeerreviewmaketitle

\section{Introduction}
\label{sec:introduction}
\IEEEPARstart{C}{omplete} motion planning remains a fundamental challenge in robotics, requiring algorithms that both \TYPO{find collision-free} paths when they exist and certify infeasibility otherwise~\cite{CELL-LABELLING, INFEAS-TD}. Infeasibility certificates are particularly important in Task and Motion Planning (TAMP) problems \cite{LEARN-FEAS, TAMP-AND/OR}, where infeasible motions necessitate replanning at the task level. Classical approaches based on exact space decomposition, such as vertical cell decomposition \cite{VCD} and visibility graphs~\cite{PLAN-LAVALLE}, provide completeness guarantees but are largely restricted to low-dimensional spaces with simple polygonal obstacles. For more complex geometries, Cylindrical Algebraic Decomposition (CAD) \cite{CAD} offers a principled framework by representing obstacles as semi-algebraic sets; however, its doubly exponential complexity in both the number of obstacles and the dimensionality of the configuration space severely limits its practical applicability.
 
In higher-dimensional settings, sampling-based planners \cite{RRT, PRM, RRT-Connect} such as Rapidly-exploring Random Trees (RRT), Probabilistic Roadmaps (PRM) and their variants are widely used due to their computational efficiency. However, these methods are only probabilistically complete, meaning that \TYPO{guarantees} of infeasibility can only be obtained in the limit of infinite sampling. Consequently, providing reliable certificates of infeasibility in high-dimensional configuration spaces remains an open problem.

Recent work has explored learning-based approaches that approximate separating manifolds between start and goal configurations, for example using Support Vector Machines (SVMs) \cite{INFEAS-TD, INFEAS-COX}. While such methods can efficiently detect infeasibility, they operate primarily as classification mechanisms and do not capture the underlying geometric structure of the configuration space. In particular, they do not identify the specific obstacle boundaries responsible for the separation.
\begin{figure}[t]
\centering
\begin{subfigure}[t]{0.49\columnwidth}
    \centering
    \includegraphics[width=0.935\linewidth]{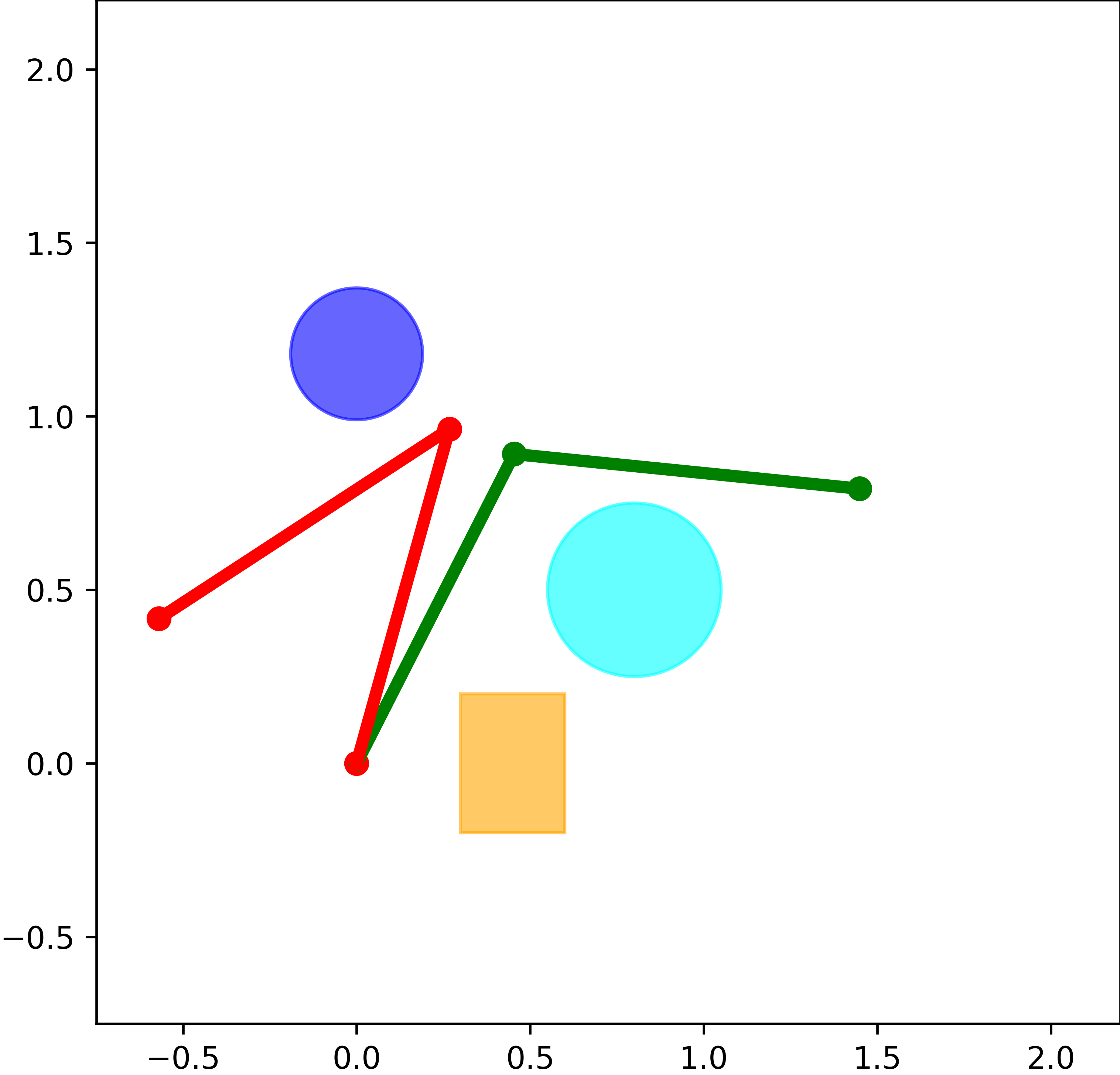}
    \caption{\TYPO{2-DOF} Robot in 2D Workspace}
    \label{fig:obstacle_boundary_2d}
\end{subfigure}
\hfill
\begin{subfigure}[t]{0.49\columnwidth}
    \centering
    \includegraphics[width=1.065\linewidth]{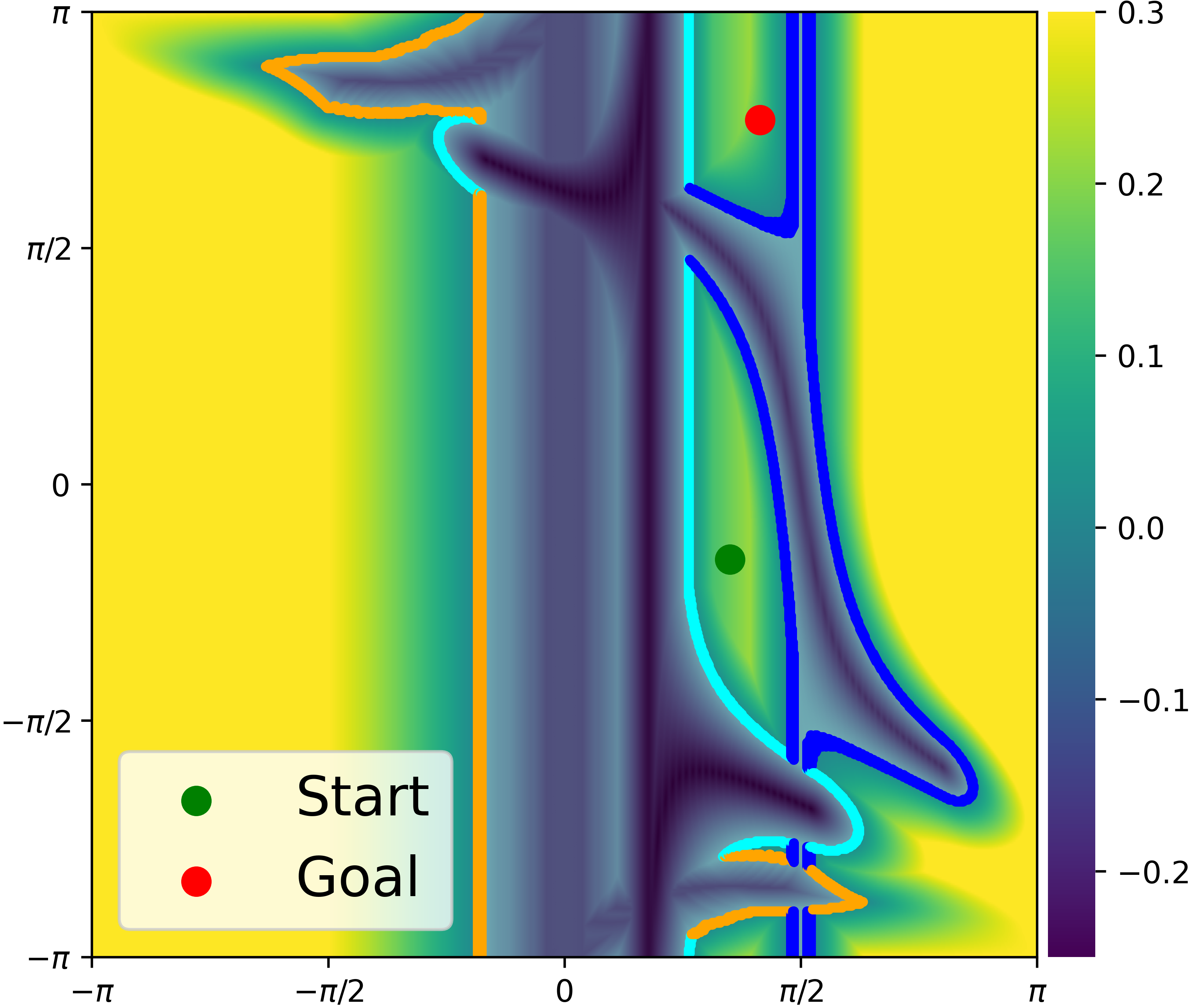}
    \caption{SDF Plot on a 2D C-Space}
    \label{fig:sdf_plot_2d}
\end{subfigure}
\caption{(a) A 2-DOF robot arm operating in a planar workspace with circular and polygonal obstacles, assuming a start configuration (green) and a goal configuration (red). (b) The Signed Distance Function (SDF) over the configuration space. The traced obstacle boundaries, corresponding to the zero level set of the SDF, are depicted in the same color as the obstacles in the workspace.} 
\label{fig:zero_sdf_2d}
\end{figure}
In this paper, we propose a novel approach for certifying infeasibility through direct geometric analysis of the configuration space. Our method traces the separating manifolds induced by obstacle boundaries using the Signed Distance Function (SDF) projected onto the configuration space, as shown in Fig.~\ref{fig:zero_sdf_2d}. This enables not only the detection of infeasibility but also the identification of the obstacles and configurations that give rise to it. This notion of \emph{actionable infeasibility} provides a structured understanding of motion planning failure by revealing the geometric causes of separation, thereby offering insights that can guide potential environment modifications to restore feasibility, as often required in Navigation Among Movable Obstacles (NAMO)~\cite{NAMO-STILMAN, THOMAS-MOD, muguiraIturralde2023ICRA} and TAMP problems.\\
Realizing this geometric perspective in higher-dimensional configuration spaces, however, requires algorithms capable of efficiently exploring and reconstructing complex manifold structures under significant computational constraints. To this end, we develop a parallelized frontier-expansion algorithm that leverages GPU acceleration to efficiently traverse the underlying geometric structure through \RR{a resolution-dependent} simplicial reconstruction. We demonstrate results on 4-DOF and 5-DOF manipulator scenarios and discuss pathways toward scaling to higher-dimensional systems.
\section{Related Work}
\label{sec:related_work}
\subsection{Space Decomposition and Cell-Based Methods}
\RR{Complete motion planning methods decompose the configuration space into 
cells whose connectivity determines path existence~\cite{NOPATH}. 
Approximate Cell Decomposition (ACD) partitions the configuration space 
into a hierarchical grid of free, occupied, and mixed cells; path 
non-existence is established when the start and goal lie in disconnected 
components of the free-cell connectivity graph. Zhang et al.~\cite{CELL-LABELLING} 
extend this framework by introducing efficient C-obstacle queries based 
on separation distance and generalized penetration depth to label cells 
as free or occupied without explicitly constructing the full obstacle 
region, providing resolution-complete path non-existence guarantees for 
robots up to 3-DOF. While effective, both approaches scale exponentially 
with configuration space dimension, limiting their applicability to 
higher-dimensional systems. Variants combining ACD with sampling-based 
planners improve handling of narrow passages while retaining resolution 
completeness~\cite{HYD-COMPLETE}, though the exponential scaling with 
dimension remains a fundamental limitation.

Deterministic sampling approaches such as star-shaped roadmaps~\cite{STAR-ROADMAP} decompose the configuration space into star-shaped regions connected through shared boundaries, providing completeness guarantees via roadmap connectivity analysis. While offering deterministic guarantees, these methods rely on explicit decomposition and become computationally intractable beyond three dimensions. Thomas et al.~\cite{THOMAS-INFEAS} take an incremental approach, discretizing the configuration space into a bitmap and sampling collision configurations until the start and goal are in distinct connected components.}
\subsection{\RR{Connectivity-Based Algorithms}}
\vspace{-1.5pt}
\RR{Another line of work focuses on the deterministic verification 
of path non-existence through caging and connectivity analysis 
of collision space \cite{CAGING}. Caging-based methods analyze 
whether an object is trapped by obstacles by studying the 
connectivity of the collision space or its lower-dimensional 
slices, providing formal guarantees of path non-existence when 
the free space is disconnected, but are typically restricted 
to specific object geometries and do not generalize to 
articulated robots in high-dimensional spaces.

Related work has examined disconnection through narrow passages 
or geometric gates \cite{DISCON-PROOFS}, wherein infeasibility 
is established by demonstrating that an object cannot traverse 
a constrained region under any admissible orientation. While 
providing geometric certificates, these methods are specialized 
to particular geometric configurations and do not scale to 
higher dimensions. McCarthy et al.~\cite{ALPHA} sample 
collision configurations weighted by generalized penetration 
depth to construct Euclidean balls guaranteed to lie within 
$C_{\mathrm{obs}}$, assembled into a weighted $\alpha$-shape 
via the Regular Triangulation. Since every interior simplex 
is fully contained within $C_{\mathrm{obs}}$ by construction, 
separation of start and goal by the exterior simplices 
constitutes a formal certificate of path non-existence, with 
asymptotic completeness as sampling density increases. 
Compared to the proposed method, \cite{ALPHA} provides 
stronger formal guarantees independent of any resolution 
parameter, but operates on the interior of $C_{\mathrm{obs}}$ 
rather than its boundary and does not attribute infeasibility 
to specific workspace obstacles.}
\subsection{Learning-Based Infeasibility Detection}
More recent approaches formulate infeasibility detection as a classification problem by learning separating surfaces between start and goal configurations. The learning phase is followed by triangulation to validate the resulting separating manifold \cite{INFEAS-SVM, INFEAS-TD, INFEAS-COX}. While these methods enable efficient detection of infeasibility in configuration spaces of up to five dimensions, the learned separating surfaces do not provide direct geometric insight into which obstacles induce the separation. Similarly,~\cite{VISUAL-FEAS} predicts feasibility of a discrete action from visual inputs; however, it likewise does not yield interpretable geometric certificates of infeasibility.

\textit{In contrast to prior work, the proposed method computes separation directly from obstacle boundaries in configuration space and identifies the specific obstacles responsible for it. This yields both a \RR{resolution-dependent infeasibility certification} and a geometric explanation of its cause, which we refer to as actionable infeasibility.} 

\section{Preliminaries} 
\label{sec:preliminaries}
\subsection{Obstacle Representation}
\begin{figure*}[t]
    \centering
    \begin{subfigure}{0.325\textwidth}
        \centering
        \includegraphics[width=\linewidth]{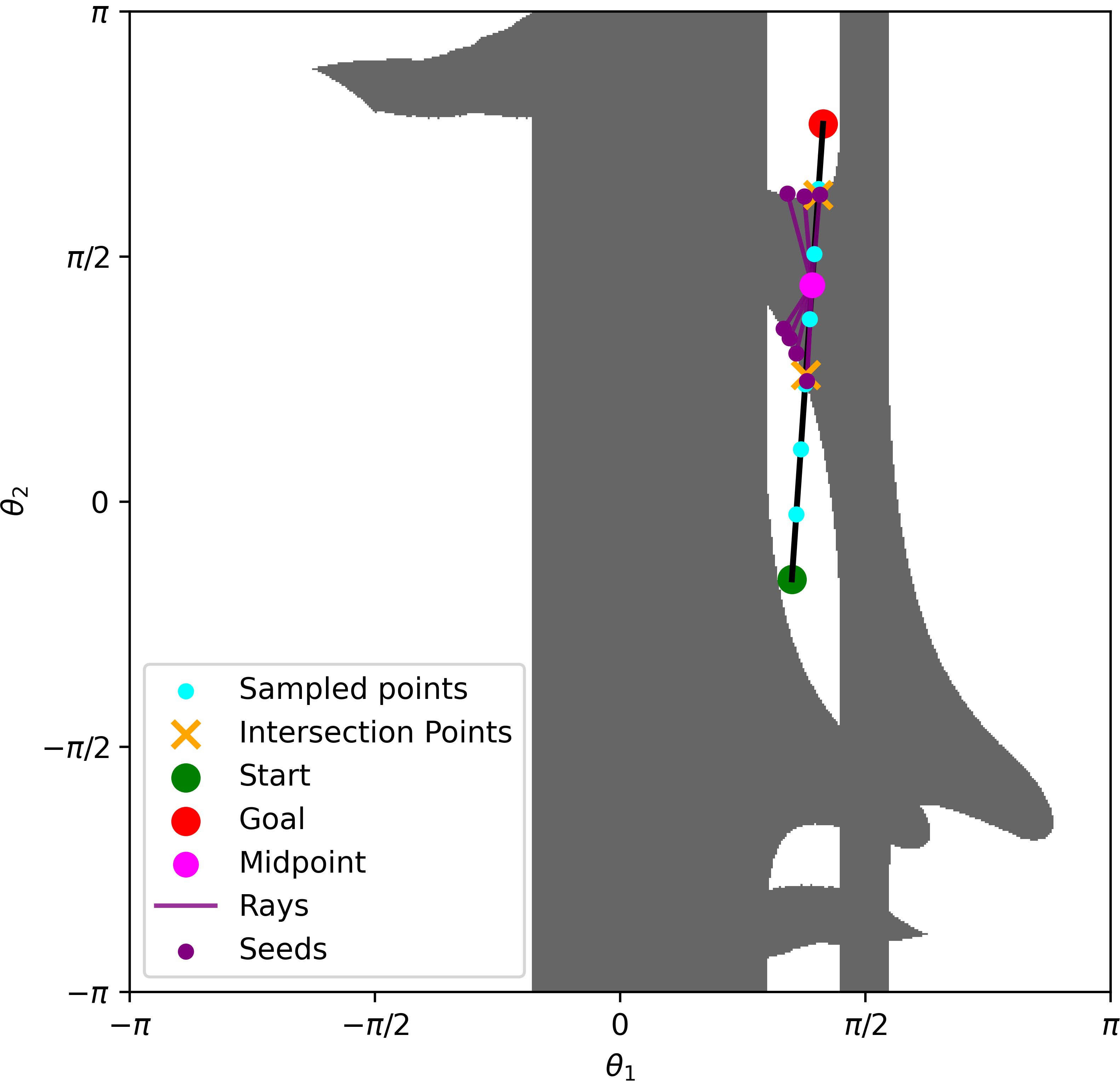}
        \vspace{-0.6cm}
        \caption{Seed sampling on the obstacle boundaries}
        \label{fig:fk_triangulation:seed}
    \end{subfigure}
    \hfill
    \begin{subfigure}{0.325\textwidth}
        \centering
        \includegraphics[width=\linewidth]{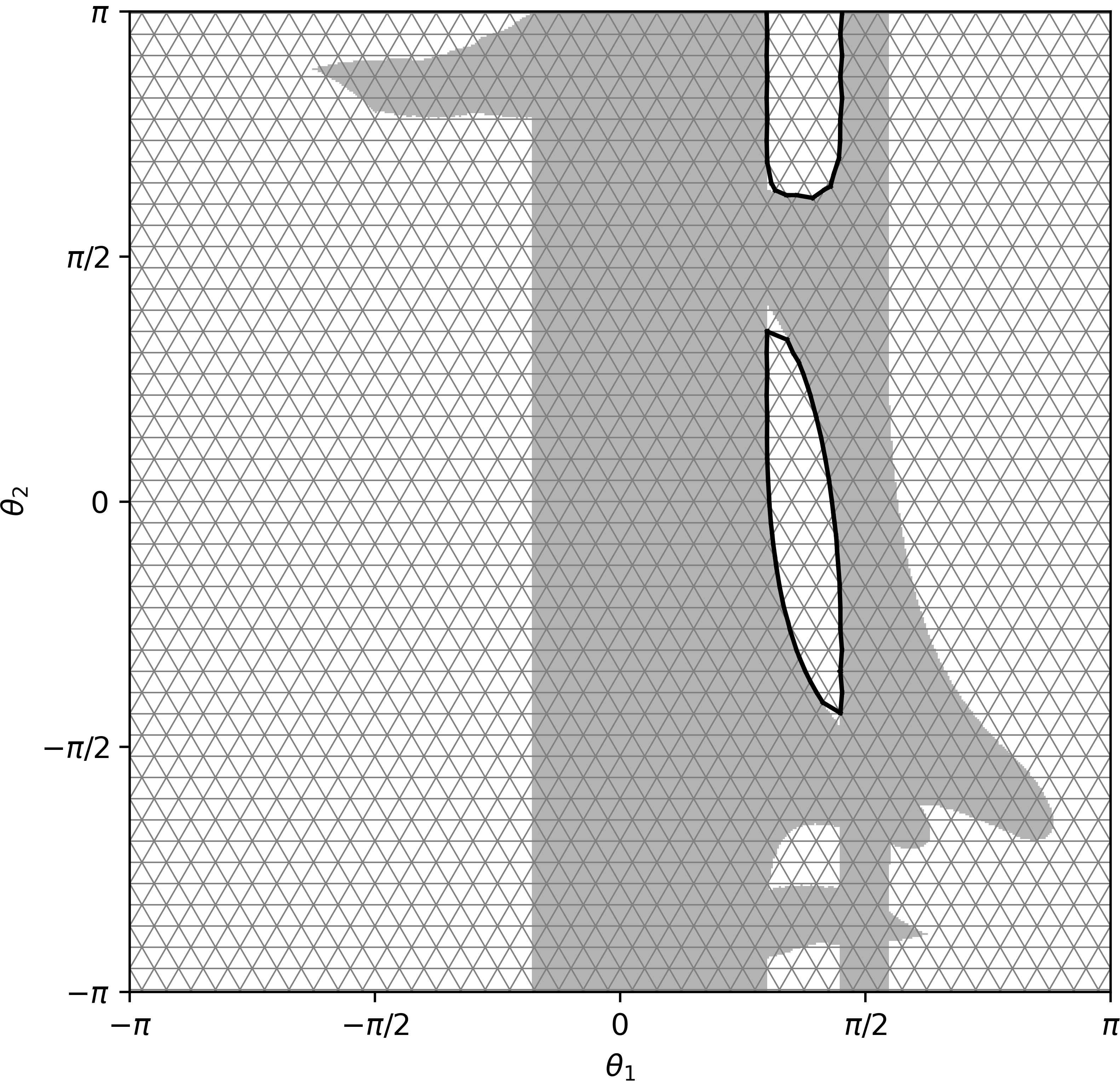}
          \vspace{-0.6cm}
        \caption{Coxeter-Triangulation for tracing}
        \label{fig:fk_triangulation:triangle}
    \end{subfigure}
    \hfill
    \begin{subfigure}{0.325\textwidth}
        \centering
        \includegraphics[width=\linewidth]{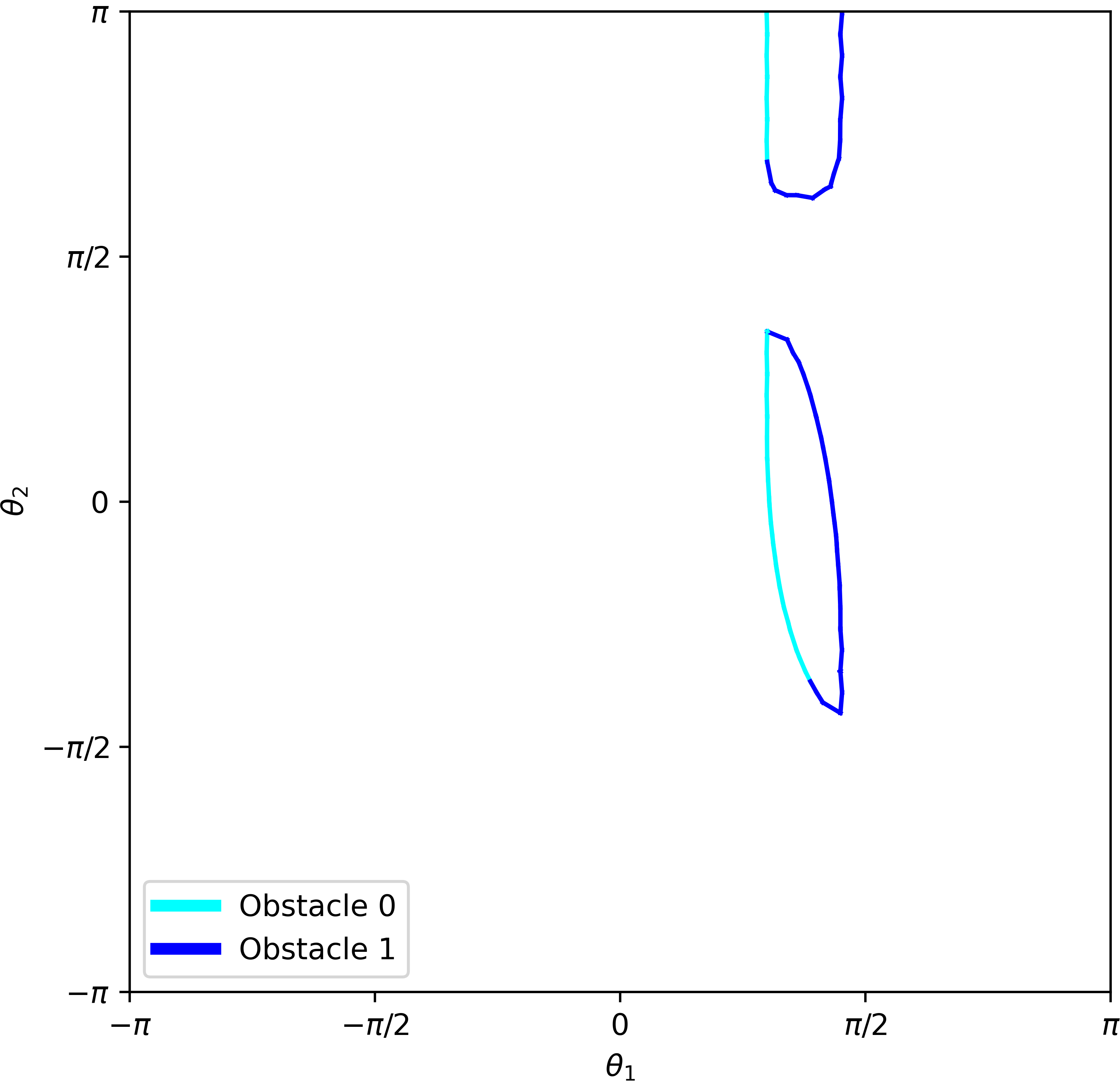}
          \vspace{-0.6cm}
        \caption{Obstacle contact manifolds}
        \label{fig:fk_triangulation:boundary}
    \end{subfigure}
      \vspace{-0.1cm}
    \caption{Overview of the proposed algorithm in a 2D configuration space corresponding to the 2-DOF robot in Fig.~\ref{fig:sdf_plot_2d}.
    (a) Seed points are generated on the configuration space obstacle boundary using the proposed strategy (Section~\ref{subsec:seed_sampling}) to avoid tracing irrelevant regions.
    (b) The obstacle boundary is traced via the \TYPO{Coxeter-Freudenthal-Kuhn} triangulation of the ambient space.
    (c) The traced boundary is decomposed into obstacle-specific components, identifying the obstacles responsible for the separation between start and goal configurations.}
    \vspace{-0.5cm}
    \label{fig:fk_triangulation}
\end{figure*}

\TYPO{The} \SHORT{obstacle region in configuration space can be represented 
using geometric constructions such as Minkowski sums \ \cite{C-SPACE-PLAN}, 
grid-based approximations \cite{OCCUPANCY-THRUN, EDT-MAP, VOXEL-MAP}, or 
implicit formulations whose zero level-sets define obstacle boundaries 
\cite{CDF, iSDF, PLAN-LAVALLE}. The latter are particularly suited to 
high-dimensional spaces where explicit construction is computationally 
prohibitive. In this work, we adopt a signed distance function representation that enables manifold tracing along obstacle contact surfaces.} \RR{ Importantly, the SDF used here differs from standard formulations as it does not satisfy the Eikonal equation. Since only the zero level-set is needed for boundary extraction, we use the projected SDF directly, which yields an identical boundary to the true configuration space distance function \cite{CDF}.}

Let each obstacle in the workspace be denoted by $X_{\R{obs}}^i \RR{\subset \mathbb{R}^m}$ for $i \in I$.
We define the signed distance to obstacle $i$ as
\[
\phi_i(\B{x}) =
\begin{cases}
d(\B{x}, X_{\R{obs}}^i), & \B{x} \notin X_{\R{obs}}^i, \\
-\, d\!\left(\B{x}, (X_{\R{obs}}^i)^c\right), & \B{x} \in X_{\R{obs}}^i,
\end{cases}
\]
where $d(\B{x},S) = \inf_{\B{y} \in S} \|\B{x} - \B{y}\|$ denotes the distance from a point to a set and \RR{$m \in \{ 2, 3\}$ is the workspace dimension}. The signed distance between the robot at any configuration $q$ and the workspace obstacles is then defined as
\begin{equation}
\R{sdf}(\B{q}) =  \min_{i \in I} \; \min_{\B{a} \in A(\B{q})} \phi_i(\B{a}) \label{eq:signed_distance}
\end{equation}
where \RR{$A(\B{q}) \subset \mathbb{R}^m$} denotes the set of points occupied by the robot at configuration $\B{q}$. The configuration-space obstacle region is defined as
\begin{equation}
C_{\mathrm{obs}} = \{ \mathbf{q} \in \mathcal{C} \mid A(\mathbf{q}) \cap X_{\mathrm{obs}} \neq \emptyset \},
\end{equation}
where $X_{\mathrm{obs}} = \bigcup_{i \in I} X_{\mathrm{obs}}^i$ denotes the union of all workspace obstacles. Equivalently, using the signed distance function in~\eqref{eq:signed_distance}, the obstacle region (\RR{including the contact region}) can be characterized as
\begin{equation}
C_{\mathrm{obs}} = \{ \mathbf{q} \in \mathcal{C} \mid \mathrm{sdf}(\mathbf{q}) \leq 0 \},
\end{equation}
with the obstacle region boundary given by the zero level set
\begin{equation}
\partial C_{\mathrm{obs}} = \{ \mathbf{q} \in \mathcal{C} \mid \mathrm{sdf}(\mathbf{q}) = 0 \}.
\end{equation}
\subsection{Actionable Infeasibility}
Given start and goal configurations $\B{s}, \B{g} \in C_{\text{free}}$, a path between them is infeasible if $\B{s} \in C_s$ and $\B{g} \in C_g$, where $C_s, C_g \subset C_{\text{free}}$ are distinct path-connected components such that $C_s \cap C_g = \emptyset $. In this case, the \RR{codimension-one (with respect to $\dim(\C{C})$)} boundary-less manifolds $\partial C_s$ and $\partial C_g$ (i.e., the boundaries of the connected components containing $\B{s}$ and $\B{g}$, respectively) \RR{are subsets of} the boundary of the obstacle region $\partial C_{\text{obs}}$, and act as separating \RR{manifolds} between the start and goal configurations. Since obstacle boundaries in configuration space may be induced by multiple workspace obstacles, $\partial C_{\text{obs}}$ can be decomposed into obstacle-induced contact manifolds.
\RR{\begin{lemma}(Infeasibility certificate by separation). Let $\C{C}$ be a topological space, let $C_{\R{obs}} \subseteq \C{C}$ be closed, and define
\begin{equation}
    C_{\R{free}} = \C{C} \setminus C_{\R{obs}}
\end{equation}
Let $s, g \in C_{\R{free}}$, and let
\begin{equation}
    S \subseteq C_{\R{obs}}
\end{equation}
be a closed set. Suppose that $s$ and $g$ lie in distinct path-connected components of $\C{C} \setminus S$. Then $s$ and $g$ lie in distinct path-connected components of $C_{\R{free}}$. Equivalently, there is no collision-free path from s to g.
\end{lemma}
\begin{proof}
Assume, for contradiction, that there exists a continuous path
\begin{equation}
\gamma : [0, 1] \rightarrow C_{\R{free}}
\end{equation}
with $\gamma(0) = s$ and $\gamma(1) = g$. Since $S \subseteq C_{\R{obs}}$, we have $C_{\R{free}} = C \setminus C_{\R{obs}} \subseteq C \setminus S$. Therefore $\gamma$ is also a path in $C \setminus S$ from $s$ to $g$, contradicting the assumption that s and g lie in distinct path-connected components of $C \setminus S$. Hence no such path exists in $C_{\R{free}}$.
\end{proof}}
\begin{definition}
Let the configuration space obstacle boundary be denoted by
\vspace{-0.3cm}
\[
\partial C_{\text{obs}} = \bigcup_{i=1}^{N} \C{M}_i,
\]
where $\C{M}_i$ is the contact manifold induced by obstacle $i$.
A separating pair of obstacle boundaries $\C{S} = \{\partial C_s, \partial C_g\}$ between $\B{s}$ and $\B{g}$ is such that every continuous path from $\B{s}$ to $\B{g}$ intersects $\C{S}$.
An \textit{actionable infeasibility proof} is the identification of a set of obstacles
$I \subseteq \{1,\dots,N\}$ such that
\[
\partial C_s \cup \partial C_g = \bigcup_{i \in I} \C{M}'_i,
\quad \text{where } \C{M}'_i \subseteq \C{M}_i,
\]
and $\C{S}$ separates $C_s$ and $C_g$. The obstacle set $I$ is said to be \RR{the obstacles on the nearest contact manifold in the separaion set} and provides an actionable explanation \RR{in the sense that iteratively evaluating the actionable infeasibility set and removing the obstacles in the set would lead to a feasible solution (optimality is not guaranteed)}. As illustrated in Fig.~\ref{fig:fk_triangulation:boundary}, this set contains the two circular obstacles in Fig.~\ref{fig:obstacle_boundary_2d}. \RR{Removing} either obstacle gives a feasible path from $\B{s}$ to $\B{g}$.
\end{definition}
\vspace{-0.6cm}
\subsection{\RR{Resolution-dependent} Manifold Tracing}
A simplicial complex provides a piecewise-linear approximation of a manifold by decomposing it into simplices. Given a $(d-1)$-dimensional manifold represented implicitly as the zero level-set of a function embedded in a $d$-dimensional ambient space, a simplicial approximation can be constructed using classical triangulation methods such as Marching Cubes and related grid-based techniques \cite{MARCHING-CUBES, RMT}. However, these methods rely on global grid discretizations, which become computationally prohibitive in higher-dimensional spaces due to the exponential growth in the number of grid cells, even when parallelized.

To address this computational bottleneck, we adopt a numerical continuation-based approach built on the Coxeter-Freudenthal–Kuhn triangulation (CFK) (see Fig. \ref{fig:fk_triangulation:triangle}). Instead of constructing a global grid, the method incrementally explores the manifold by traversing adjacent simplices that intersect the zero level-set in the ambient space. This strategy prevents redundant grid computations and enables efficient tracing of the manifold by focusing computation only in regions that intersect the implicit \RR{hypersurface} using a frontier expansion style traversal \cite{MESH-COX} which makes it parallelizable. The CFK triangulation, together with the Permutahedral representation of simplices, enables efficient combinatorial queries for enumerating lower-dimensional faces and higher-dimensional cofaces, allowing traversal across adjacent simplices intersecting the manifold. The resulting piecewise-linear approximation is controlled by a scale parameter $\lambda$, which determines the triangulation resolution.
\vspace{-0.45cm}
\subsection{Assumptions}
We make the following assumptions in this approach. First, we assume a Euclidean configuration space, as manifold tracing via triangulation requires a Euclidean space. This is a reasonable assumption for physical manipulators with joint limits. Second, we focus on kinematic infeasibility, i.e., infeasibility arising from $C_{\mathrm{obs}}$, which is a reasonable assumption in many practical robotic manipulation settings where obstacle constraints dominate. 


%
\vspace{-2mm}
\section{Algorithm}
\label{sec:approach}

In this section, we present our approach for constructing actionable 
infeasibility proofs in configuration space via a parallel, GPU-accelerated pipeline. Given the robot geometry, kinematic model, and static workspace obstacles, we compute the SDF in~\eqref{eq:signed_distance} to obtain an implicit representation of $\partial \mathcal{C}_{\mathrm{obs}}$ as the zero level-set $\R{sdf}(\B{q}) = 0$, where $\B{q} \in C$. Seed points sampled on this level-set serve as initialization for manifold tracing, which explicitly 
recovers the obstacle-induced separating manifolds in configuration space. 
\RR{The resolution of this tracing is governed by the scaling parameter 
$\lambda$: smaller values yield finer approximations but increase computational 
cost, while larger values improve speed at the risk of missing narrow passages. 
Consequently, the method is \MR{resolution-dependent }--- infeasibility is guaranteed 
to be detected only when the separating manifold is sufficiently wide relative 
to $\lambda$. We argue this is a practical necessity in high-dimensional spaces 
where exact manifold tracing is computationally intractable, and $\lambda$ can 
be tuned to balance fidelity and runtime for a given application (see Section~\ref{sec:res-comp}).}

Given a query pair $\B{s}$ and $\B{g}$, \TYPO{we determine connectivity} using a ray-based parity test with respect to the traced manifold representation. This test classifies whether $\B{s}$ and $\B{g}$ lie in distinct path-connected components of $C_{\text{free}}$. If disconnected, the method returns a \RR{resolution-dependent} infeasibility certificate together with a subset of obstacles whose induced contact manifolds form a separating boundary between the two configurations. These boundaries provide an actionable explanation of infeasibility in terms of constraint-inducing obstacle interactions. Fig. \ref{fig:fk_triangulation} summarizes the above procedure.

\subsection{Seed Sampling}
\label{subsec:seed_sampling}
The following observation holds: To determine infeasibility in configuration space, it is not necessary to reconstruct the entire obstacle boundary; instead it suffices to identify the subset of obstacle boundaries that separates $\B{s}$ and $\B{g}$. Interpreting infeasibility as the dual of motion planning, we note that if no feasible path exists between $\B{s}$ and $\B{g}$, then every continuous path connecting them must intersect the configuration space obstacle boundary. In particular, the \TYPO{straight-line segment} between $\B{s}$ and $\B{g}$ must also intersect this boundary. Consequently, the obstacle boundary components intersected by this segment must contain at least one pair of separating manifolds that certify infeasibility.  

Based on this observation, we first compute the intersection points between the straight-line path and the configuration space obstacle boundary. Algorithm~\ref{alg:seed_generation} describes our seed sampling strategy. We uniformly sample points along the line segment connecting $\B{s}$ and $\B{g}$ such that \RR{each of the sample is within some $\epsilon < \lambda$ distance of the other nearest sample to reduce the chances of missing sign changes of the SDF over the line-segment. Note that if the obstacle region across the line is thinner that $\epsilon$, Algorithm~\ref{alg:seed_generation} might miss the intersection points.} Each detected interval is then refined using the regula falsi method to obtain accurate intersection points (line~\ref{alg:seed_generation:line1}). 

To enable efficient parallel manifold tracing, we further generate additional seed points on the boundary by casting rays from the midpoints of pairs of intersection points in randomly chosen directions (lines~\ref{alg:seed_generation:line2}-\ref{alg:seed_generation:line5}). The resulting ray–boundary intersection points are computed using a Newton-based root-finding scheme with bisection-based refinement to ensure robustness (line~\ref{alg:seed_generation:line6}). These boundary points are added to the seed set $S$ for parallel tracing of the separating manifolds.


\begin{algorithm}[t]
\caption{Seed Sampling}
\label{alg:seed_generation}
\begin{algorithmic}[1]
\Require{Start $\B{s}$, goal $\B{g}$, number of rays $N_r$}
\Ensure{Seed set $S \ | \ \forall \B{q} \in S$, $\R{sdf}(\B{q})=0$}
\State{Compute line-boundary intersections $S$} \label{alg:seed_generation:line1}
\State{Compute midpoints $\B{m_k}$ between intersection pairs\label{alg:seed_generation:line2}}
\For{each midpoint $\B{m_k}$}
    \For{$r = 1$ to $N_r$ \textbf{in parallel}}
        \State{Cast ray from $\B{m_k}$ in a sampled direction $\B{d}$\label{alg:seed_generation:line5}}
        \State{Find $\R{sdf}(\B{q})=0$ \texttt{//} \textcolor{mygray}{$\B{q}$ computed using Newton-based root-finding along with Bisection method} \label{alg:seed_generation:line6} }
               \State{Add intersection to $S$}
    \EndFor
\EndFor
\State{\Return{ $S$}}
\end{algorithmic}
\end{algorithm}
\vspace{-5pt}
\subsection{Connected Components Tracing}
We build upon the simplicial-complex reconstruction framework of \cite{MESH-COX} to represent and traverse obstacle boundaries in the configuration space. Algorithm~\ref{alg:frontier_ini} presents the pseudo-code for initializing the frontier arrays with the initial simplices in the Permutahedral Representation, using the set $S$ obtained from Algorithm~\ref{alg:seed_generation}. We then use Algorithm~\ref{alg:trace} to extract all the 1-simplices (1D edges) intersecting the obstacle boundary. 

The sampled seed points are first used to identify the maximal simplices that contain them via \verb|locate_simplex| which returns simplices in Permutahedral representation (Algorithm~\ref{alg:frontier_ini}, line~\ref{alg:frontier_ini:locate_simplex}). The 1D edges of these simplices are then inserted into an initial frontier set (Algorithm~\ref{alg:frontier_ini}, line~\ref{alg:frontier_ini:frontier_add}) and recorded in a global hash table together with their corresponding intersection points (Algorithm~\ref{alg:frontier_ini}, line 7). This initialization yields the frontier array $\mathcal{F}$, which serves as the starting point for manifold tracing (Algorithm~\ref{alg:trace}).

We perform a frontier-based expansion to trace the obstacle boundary. For each edge $e \in \mathcal{F}$ (Algorithm~\ref{alg:trace}, line~\ref{alg:trace:frontier_expansion}), we enumerate its adjacent edges $e'$ via the 2D cofaces incident $e$ (Algorithm~\ref{alg:trace}, lines~\ref{alg:trace:cofaces}–\ref{alg:trace:edges_of_cofaces}). Each edge is incident to $d!$ such cofaces, where $d$ is the dimension of the configuration space, yielding candidate edges (excluding the shared edge) that are tested for intersection with the boundary. Intersecting edges are inserted into the next frontier $\mathcal{N}$ and recorded in the hash table $\mathcal{H}$ along with their intersection points (Algorithm~\ref{alg:trace}, lines~\ref{alg:trace:edge_check}–\ref{alg:trace:end_of_edge_check}). This process is repeated iteratively, enabling local, incremental reconstruction of the boundary manifold.
\begin{algorithm}[t]
\caption{Initialize Frontier}
\label{alg:frontier_ini}
\begin{algorithmic}[1]
\Require{Seed set $S$, signed distance function $\mathrm{sdf}(\B{q})$}
\Ensure{Initialized Frontier array}
\State{Initialize frontier arrays $\C{F}, \C{N} \leftarrow \emptyset$ \label{alg:frontier_ini:ini}}
\ForAll{seed points $\B{s_i} \in S$ in parallel}
    \State{\verb|locate_simplex|($\B{s_i}$) \label{alg:frontier_ini:locate_simplex}}
    \ForAll{edges $e$ of $\sigma$}
        \If{$e$ intersects $\R{sdf}(\B{q})=0$} \label{alg:frontier_ini:edge_check}
                        \State{$\B{q}_{int} \leftarrow$ \verb|intersection_point|$(\R{sdf}, e)$}
                        \State{$c \leftarrow \C{H}$.\verb|cond_hash_insert|$(e, \B{q}_{int}, i)$  \label{alg:frontier_ini:cond_hash_insert}}
            \If{$c == \R{flag1}$} \label{alg:frontier_ini:flag1}
                \State{Add $(e, i)$ to frontier $\C{F}$ \texttt{//} \textcolor{mygray}{$e$ not visited}\label{alg:frontier_ini:frontier_add}}
            \ElsIf{$c == \R{flag2}$} \label{alg:frontier_ini:flag2}
                \State{continue \texttt{//} \textcolor{mygray}{$e$ already visited}}
            \Else{\label{alg:frontier_ini:flag3}}
                \State{\verb|merge_components|$(c,i)$}
            \EndIf
        \EndIf \label{alg:frontier_ini:end_of_edge_check}
    \EndFor
\EndFor
\State{\Return{$\C{F}, \C{N}$}}
\end{algorithmic}
\end{algorithm}

Both the initialization and expansion stages rely on a conditional hash insertion procedure (\verb|cond_hash_insert|) to manage edge visitation. We employ linear probing with an FNV1a hash function \cite{FNV} for efficient storage and lookup. If an edge has not been previously visited, it is inserted into the hash table and \textit{flag1} is returned, triggering its inclusion in the current or next frontier (Algorithm~\ref{alg:frontier_ini}, line~\ref{alg:frontier_ini:flag1} and Algorithm~\ref{alg:trace}, line~\ref{alg:trace:flag1}). If the edge has already been visited from the same seed component (\textit{flag2}), the insertion is skipped (Algorithm~\ref{alg:frontier_ini}, line~\ref{alg:frontier_ini:flag2} and Algorithm~\ref{alg:trace}, line~\ref{alg:trace:flag2}). When the edge has been visited from a different component, the corresponding component identifiers are merged using \verb|merge_components|, ensuring consistent labeling across connected regions (Algorithm~\ref{alg:frontier_ini}, line~\ref{alg:frontier_ini:flag3} and Algorithm~\ref{alg:trace}, line~\ref{alg:trace:flag3}).

We note that, since frontier expansion is parallelized across edges, hash insertions and component merges are performed atomically to avoid race conditions. This design enables efficient exploration of the manifold boundary while preserving correctness under parallel execution.
\subsection{Reduced Visited Set}
To reduce memory consumption and mitigate the cost of membership checks as the hash table becomes saturated, we employ a reduced visited-set representation that retains only the current and previous frontiers, thereby preventing backtracking \cite{DCBDS, DCFS}. This design significantly lowers the GPU memory footprint and allows visited edges to be offloaded to CPU memory after each frontier expansion during manifold tracing. In addition to improving space efficiency, it also reduces insertion overhead by limiting repeated probing, leading to improved performance in high-dimensional settings.
\begin{algorithm}[t]
\caption{Manifold Tracing}
\label{alg:trace}
\begin{algorithmic}[1]
\Require{Seed points $S$, signed distance function $\mathrm{sdf}(\B{q})$}
\Ensure{Traced Edges of the Ambient Space}
\State{$\C{F}, \C{N} \leftarrow $\verb|initialize_frontiers|($S, \R{sdf}$) \label{alg:trace:line1}}
\State{Initialize Hash Table $\C{H} \leftarrow \emptyset$ \label{alg:trace:line2}}
\While{$\C{F}$ is not empty}
    \ForAll{$e \in \C{F}$ \textbf{in parallel}} \label{alg:trace:frontier_expansion}
        \ForAll{$c \in e.\R{cofaces}(2)$} \label{alg:trace:cofaces}
            \ForAll{edges $e'$ of c} \label{alg:trace:edges_of_cofaces}
                    \If{$e'$ intersects $\R{sdf}(\B{q})=0$} \label{alg:trace:edge_check}
                        \State{$\B{q}_{int}$$\leftarrow$\verb|intersection_point|$(\R{sdf}, e')$}
                        \State{$c$$\leftarrow$$\C{H}$.\verb|cond_hash_insert|$(e', \B{q}_{int}, i)$}
                        \If{$c == \R{flag1}$} \label{alg:trace:flag1}
                            \State{Add $(e', i)$ to next frontier $\C{N}$ \label{alg:trace:frontier_add}}
                        \ElsIf{$c == \R{flag2}$} \label{alg:trace:flag2}
                            \State{continue}
                        \Else \label{alg:trace:flag3}
                            \State{\verb|merge_components|$(c,i)$}
                        \EndIf
                    \EndIf \label{alg:trace:end_of_edge_check}
            \EndFor
        \EndFor
    \EndFor 
    \State{$\C{F} \leftarrow \C{N}$}
    \State{$\C{N} \leftarrow \emptyset$}
\EndWhile
\State{\Return{Visited Edge Set $\C{H}$}}
\end{algorithmic}
\end{algorithm}
\vspace{-0.4cm}
\subsection{Connected Component Query}

Once the set of boundary-intersecting edges is obtained, we determine whether 
$\B{s}$ \TYPO{and} $\B{g}$ lie in separate path-connected components via a 
ray-casting strategy. A ray parallel to the first configuration axis is cast 
from the query point, and its intersections with the obstacle boundary are 
retrieved from the stored edge intersections using the same Newton--bisection 
procedure as Algorithm~\ref{alg:seed_generation}. The component membership of 
the query point is then determined by the parity of the total intersection 
count, using the component identifier of the nearest intersecting edge. This 
avoids explicit triangulation of the ambient-space manifold, reducing 
computational overhead.

\SHORT{\noindent \textbf{Actionable Infeasibility:} Each SDF query returns the 
identifier of the obstacle closest to the robot, which is recorded at every 
edge--obstacle intersection point. Aggregating these identifiers across all 
intersection points yields the subset of obstacles responsible for the 
infeasibility.}
\section{Experiments}
\label{sec:experiments}
We evaluate the proposed algorithm on simulated 4-DOF and 5-DOF robotic systems across infeasible planning instances\footnote{\RR{Since it is difficult to obtain exact ground truth for plan infeasibility in high dimensions, we assume infeasibility with strong confidence determined by running PRM \cite{PRM} using the OMPL Library \cite{OMPL} with Gaussian sampling strategy \cite{GAUSS_SAMP} to address narrow passages for 30 minutes similar to \cite{INFEAS-TD}.}}. The considered scenarios are inspired by those presented in~\cite{INFEAS-COX}. The algorithm is implemented primarily in CUDA C++ and to exploit parallelism, our implementation leverages GPU acceleration via a custom geometry engine, in which signed distance queries are executed as device-level functions. Robot and obstacle geometries are modeled using primitive shapes (boxes, spheres, and cylinders) and visualized through an OpenGL-based module.  A kinematic tree data structure is used to efficiently perform forward kinematics, enabling accurate signed distance evaluations. Additionally, mesh-based distance functions are incorporated following \cite{IMPROVED-GJK}. In all scenarios, the triangulation resolution is fixed at $\lambda = 0.1$.


All experiments are conducted on a system equipped with a 16-core Intel Ultra 9 185H CPU (16 GB RAM) and an NVIDIA RTX 4060 GPU (8 GB VRAM). Reported runtimes correspond to the averages over ten independent trials for each scenario. For seed generation Algorithm~\ref{alg:seed_generation}), we use 1000 uniformly sampled points along the straight-line segment between $\B{s}$ and $\B{g}$ to find intersection points.

\begin{table}[t]
\caption{Runtime results. \textit{Tracing}- manifold tracing including Algorithms~\ref{alg:frontier_ini} and~\ref{alg:trace}; \textit{Seed}- the seed sampling time and \RR{ \textit{SVM}- the total execution time for the infeasiblity framework proposed in \cite{INFEAS-COX}.}}
\vspace{-0.2cm}
\label{tab:planning_time}
\centering
\setlength{\tabcolsep}{2.5pt}
\begin{tabular}{lccccc}
\hline
Environments & Tracing(s) & Seed(s) & Total(s) & \RR{SVM(s)~\cite{INFEAS-COX}}\\
\hline
\hline
\noalign{\vspace{3pt}}
4DOF(link 2) & 5.12$\boldsymbol{\pm}$0.03 & 0.50 & 5.62$\boldsymbol{\pm}$0.02 & \RR{52.63$\boldsymbol{\pm}$4.90}\\
4DOF(link 3) & 3.05$\boldsymbol{\pm}$0.01 & 0.23 & 3.28$\boldsymbol{\pm}$0.01 & \RR{31.15$\boldsymbol{\pm}$8.19}\\
Shoulder-Joint & 2.33$\boldsymbol{\pm}$0.03 & 0.23 & 2.56$\boldsymbol{\pm}$0.02 & \RR{39.20$\boldsymbol{\pm}$4.59}\\
SCARA(small) & 6.44$\boldsymbol{\pm}$0.02 & 0.34 & 6.79$\boldsymbol{\pm}$0.02 & \RR{74.40$\boldsymbol{\pm}$7.24}\\
SCARA(big) & 4.99$\boldsymbol{\pm}$0.03 & 0.35 & 5.35$\boldsymbol{\pm}$0.03 & \RR{80.00$\boldsymbol{\pm}$9.51}\\
\hline
\noalign{\vspace{3pt}}
Universal & 81.49$\boldsymbol{\pm}$0.29 & 0.20 & 81.69$\boldsymbol{\pm}$0.29 & \RR{253.33$\boldsymbol{\pm}54.10$}\\
Packbot (near) & 261.04$\boldsymbol{\pm}$0.86 & 0.16 & 261.20$\boldsymbol{\pm}$0.86 & \RR{494.82$\boldsymbol{\pm}$173.16}\\
Packbot (far) & 253.71$\boldsymbol{\pm}$0.79 & 0.08 & 253.79$\boldsymbol{\pm}$0.80 & \RR{433.22$\boldsymbol{\pm}$154.13}\\
\hline
\end{tabular}
\vspace{-0.2cm}
\end{table}

\noindent\textbf{4-DOF Robot Scenarios:}
We evaluate the performance of the proposed approach across three different 4-DOF robot scenarios: (a) a 4-DOF manipulator trying to reach inside a frame (Fig.~\ref{fig:sim:4dofarm}); (b) a two-link arm with shoulder and elbow joints (Fig.~\ref{fig:sim:shoulderbot}); and (c)  a SCARA-like robot with three revolute and one prismatic joint, where a green block attached \TYPO{to} the end-effector must reach inside a box (Fig.~\ref{fig:sim:scara}). For both the 4-DOF manipulator and the shoulder–elbow system, infeasibility is induced by placing the goal  within a rectangular frame, rendering it unreachable. In the SCARA-like setup, the goal requires the green block to pass through a narrow slot to reach the interior of a box; a spherical obstacle positioned in front of the box makes this infeasible. For the 4-DOF manipulator, we vary the robot base position, while for the SCARA scenario, we enlarge the obstacle to generate additional scenes. These variations are discussed in Section~\ref{subsec:boundary_size}.

The first five rows of Table~\ref{tab:planning_time} report the results. Across all experiments, the proposed method consistently detects infeasibility within a few seconds by tracing obstacle boundaries in the configuration space. This process enables clear identification of the obstacle regions responsible for infeasibility (red obstacles in Fig.~\ref{fig:experiments}), which manifest as disjoint boundary components. The number of sampled rays is set to 100, which empirically provides the best performance across all considered 4-DOF scenarios.
 
\noindent\textbf{5-DOF Robot Scenarios:}
We conduct 5-DOF experiments using robot geometries and kinematic structures inspired by the Universal UR5 \cite{UNIVERSAL} and the PackBot 510 manipulator~\cite{PACKBOT}, as illustrated in Fig. \ref{fig:sim:universal} and Fig. \ref{fig:sim:packbot}, respectively. In the PackBot setup, seed generation is restricted to straight-line intersection points, whereas for the UR5-inspired system, we employ 10 sampled rays. The PackBot environment consists of a single table as the primary obstacle; the robot is initialized beneath the table and tasked with reaching a configuration above it, which is infeasible. In contrast, the UR5-inspired scene \TYPO{contains} multiple primitive-shaped obstacles arranged on a tabletop, requiring the end effector to navigate through them. The arrangement of obstacles (red obstacles in Fig.~\ref{fig:sim:packbot}) renders the target configuration infeasible. Results are reported in Table~\ref{tab:planning_time} (last three rows). The Universal robot scenario completes in under one and a half minutes, while the two PackBot scenes with varying table heights (see Section~\ref{subsec:boundary_size}) require approximately four and a half minutes. This difference is primarily due to the larger obstacle boundary that must be traced in the PackBot scenes.

\RR{\noindent\textbf{SVM-Infeasibililty:} We compare our execution-time with the algorithm proposed in \cite{INFEAS-COX} using the same parameter values for SVM as suggested in the paper and $\lambda = 0.1$\footnote{\RR{Our best-effort implementation of \cite{INFEAS-COX} is available at \href{https://anonymous.4open.science/r/SVM_Infeasibility-85DB}{\textit{anonymous-repo}}}.}. We can see that our infeasibility detection is $\sim$10x faster for 4DOF cases and $\sim$2x faster for 5DOF cases while also managing to capture the geometric structure of contact manifolds in $\C{C}$-space.  }


\begin{figure}[t]
    \centering
    \includegraphics[width=1.0\linewidth]{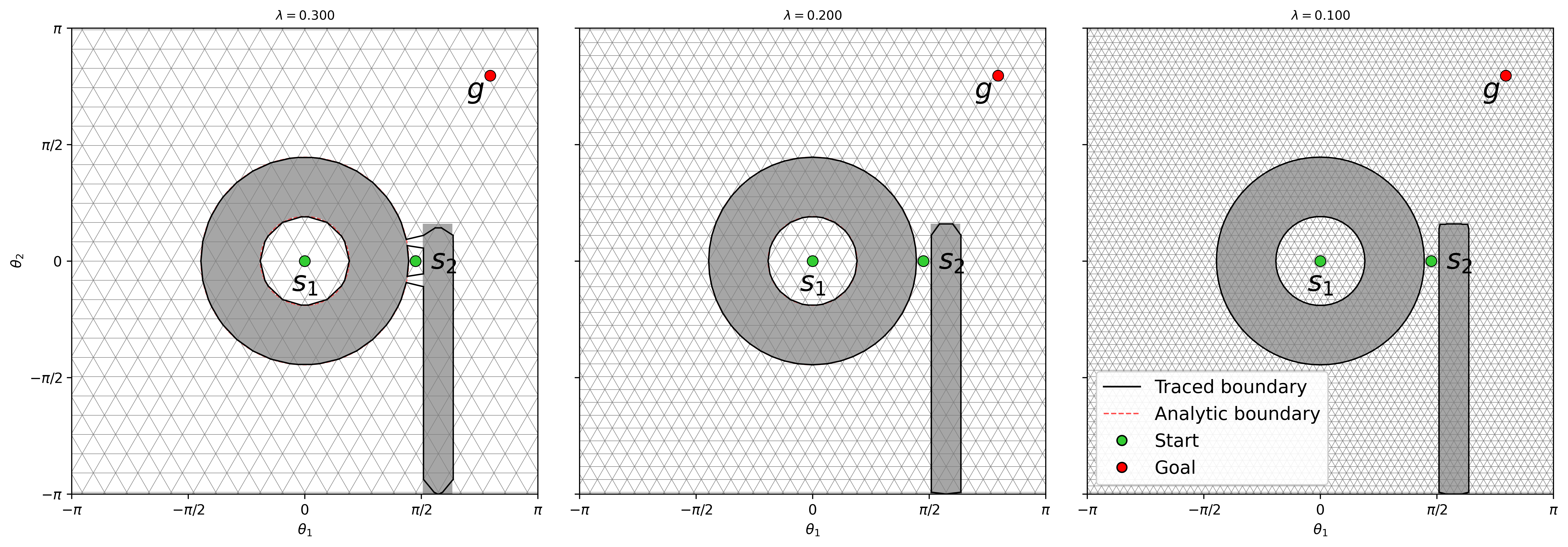}
    \vspace{-0.6cm}
    \caption{\RR{Ablation on $\lambda$ values: A lower $\lambda = 0.1$ has more accuracy whereas a higher $\lambda = 0.3$ can trace narrow regions incorrectly, resulting in false infeasibility detection for $s_2$ and $g$.}}
    \vspace{-0.3cm}
    \label{fig:analytic}
\end{figure}

\begin{figure*}[h]
\centering
\subfloat[4-DOF Robotic Arm.]{\includegraphics[trim=480 0 600 0,clip,scale=0.093]{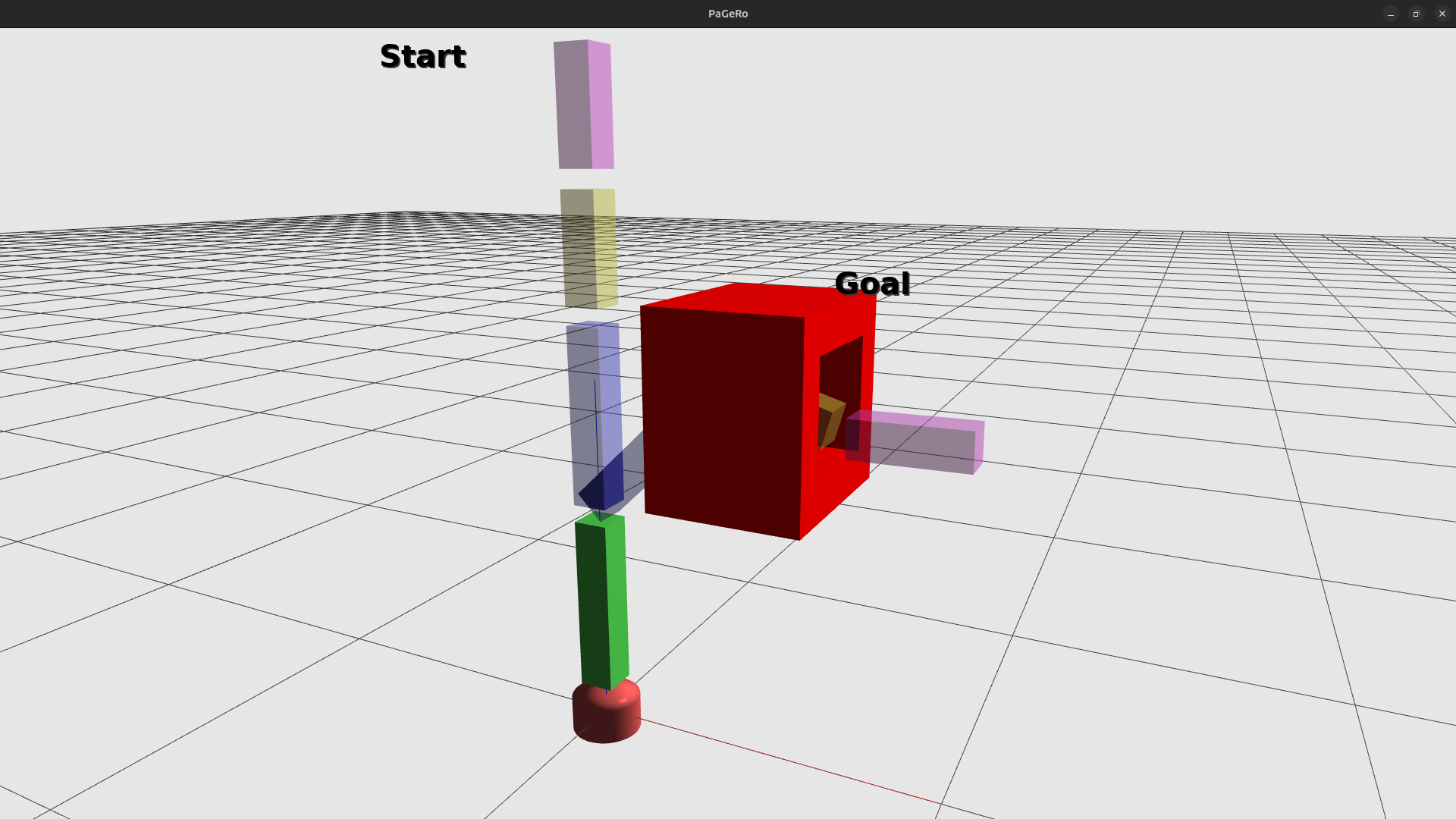} \label{fig:sim:4dofarm}}
\subfloat[Shoulder Joint.]{\includegraphics[trim=400 0 650 0,clip,scale=0.093]{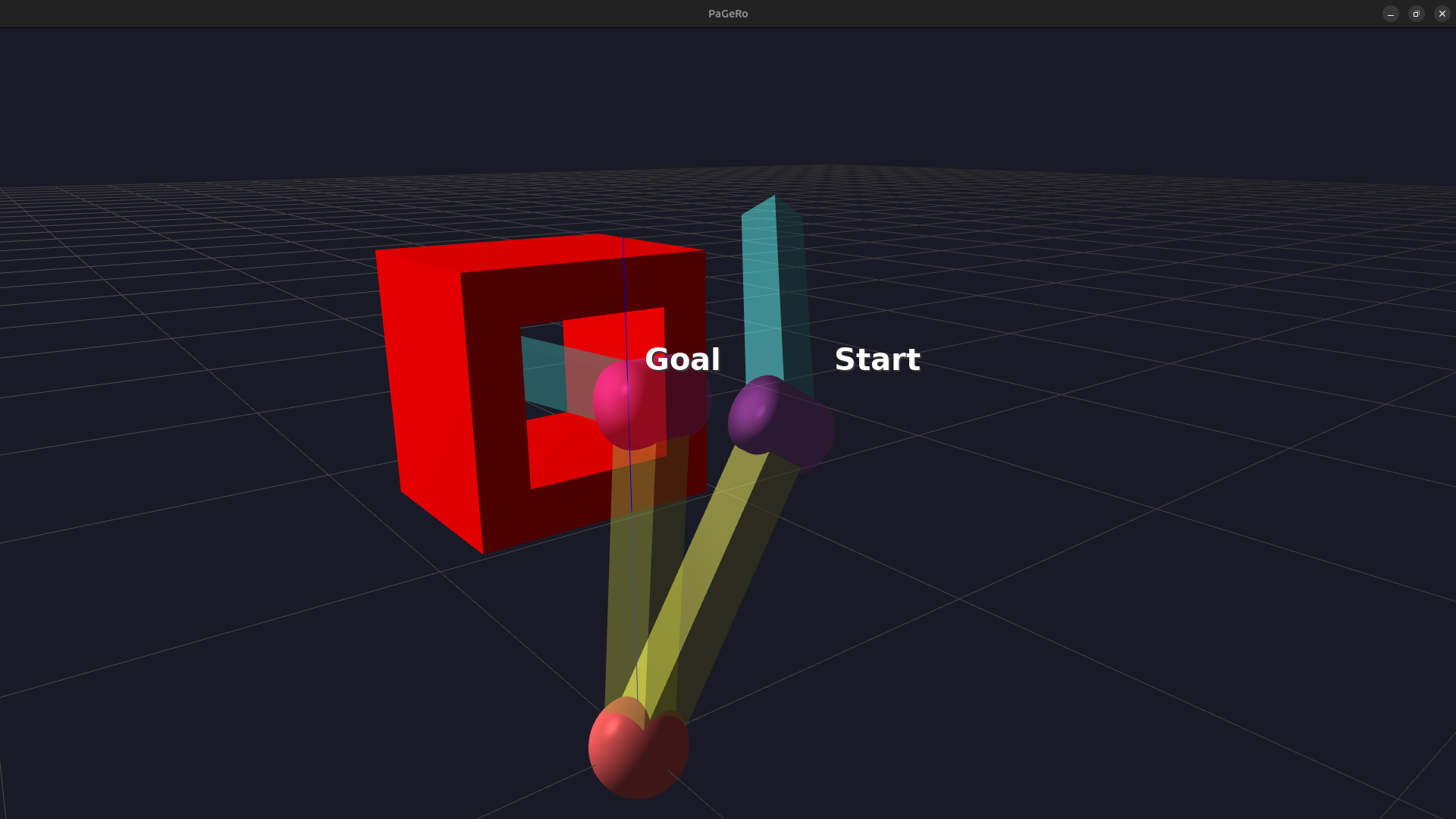} \label{fig:sim:shoulderbot}}
\subfloat[SCARA-like.]{\includegraphics[trim=600 0 600 0,clip,scale=0.093]{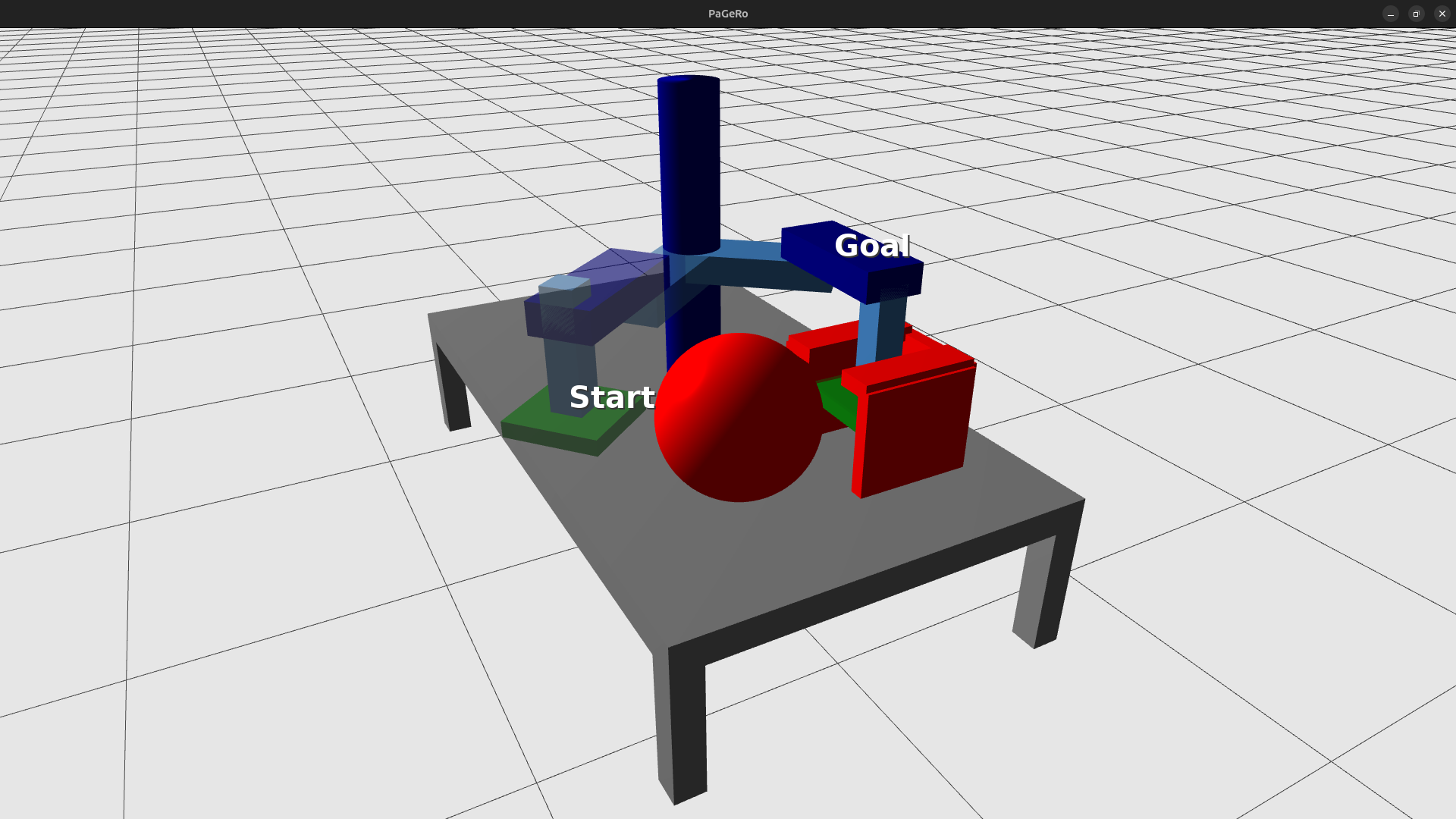} \label{fig:sim:scara}}
\subfloat[Universal scenario.]{\includegraphics[trim=150 0 300 0,clip,scale=0.093]{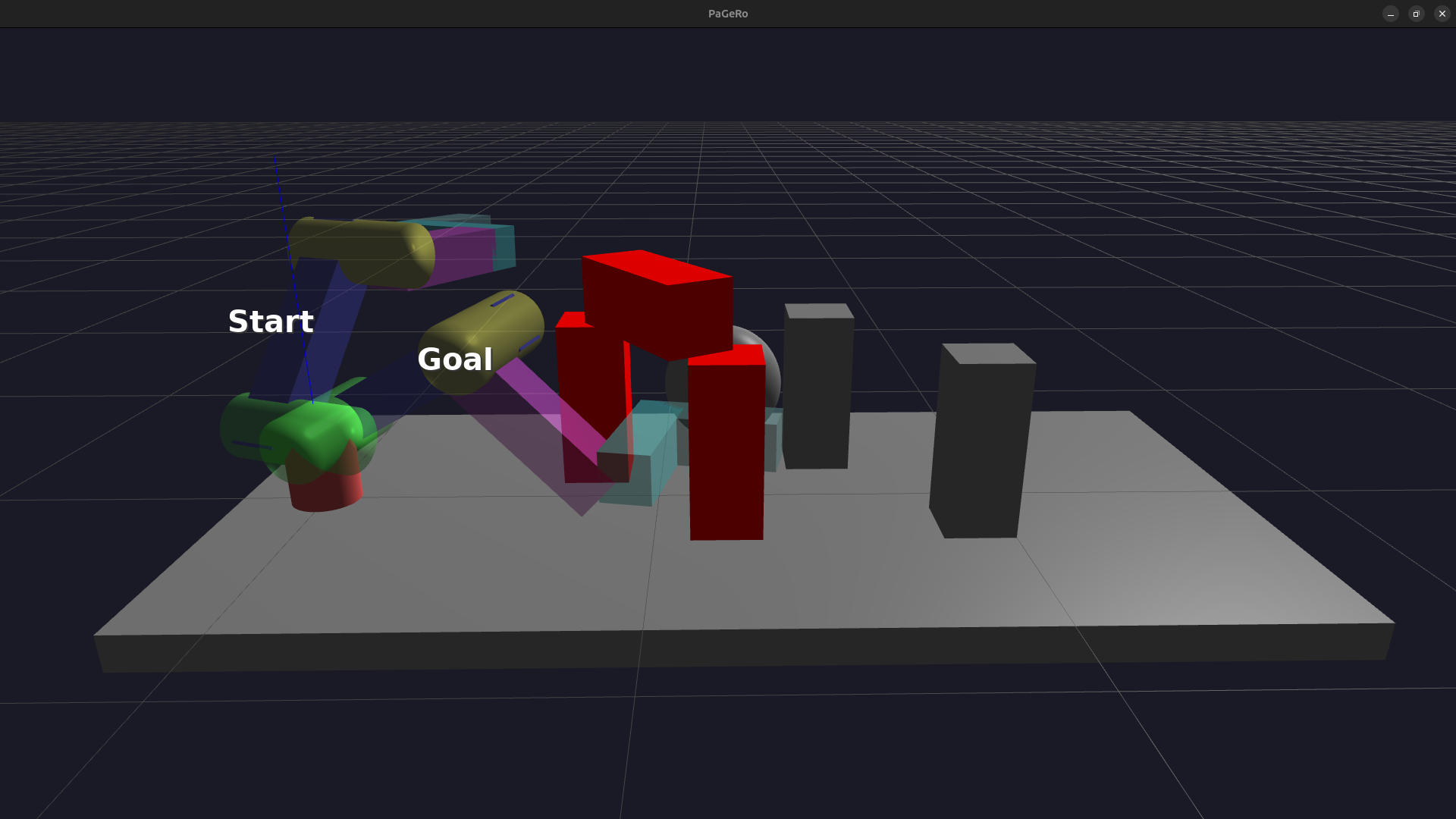} \label{fig:sim:universal}}
\subfloat[PackBot scenario.]{\includegraphics[trim=200 0 280 0,clip,scale=0.093]{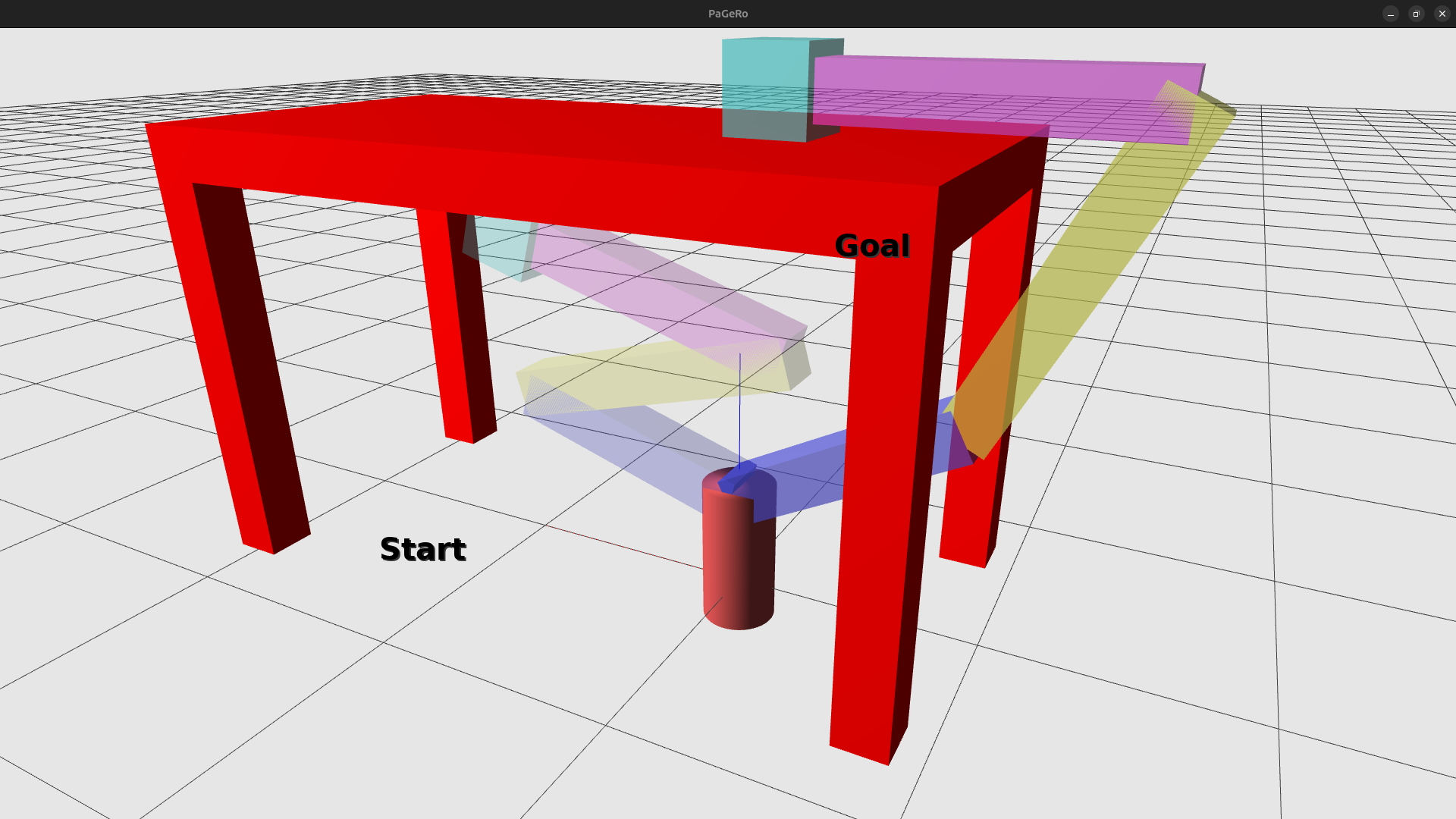} \label{fig:sim:packbot}}
\vspace{-0.2cm}
\caption{Experiment scenarios. The obstacles causing infeasibility have been highlighted in red.}
\label{fig:experiments}
\vspace{-0.5cm}
\end{figure*}

\section{Discussion, Limitations \& Future Work}
\label{sec:discussion_updated}
\MR{\subsection{Resolution-Dependence}}
\label{sec:res-comp}

The accuracy of the manifold tracing is governed by $\lambda$, which 
must be small enough to avoid missing obstacle boundaries while 
remaining large enough to ensure termination within memory constraints. 
A sign change across a simplex edge is guaranteed to be detected 
when the $\C{C}_{obs}$ width exceeds the circumradius of the 
$\tilde{A}_d$ triangulation $R = \sqrt{\frac{d(d+2)}{12(d+1)}}  \lambda =c \lambda$ \cite{MESH-COX}. We compute a heursitic maximum bound $\lambda_{max} = \rho_{\min}/ (c\sigma_{\text{max}}^{\text{global}})$ where $\rho_{\min}$ is min workspace width and $\sigma_{\max}^{\text{global}}$ is the maximum spectral norm of Jacobian. The effect of $\lambda$ on tracing accuracy is demonstrated in Figure~\ref{fig:analytic} for an analytically defined obstacle region in the configuration space. A formal proof that the $\lambda_{\max}$ guarantees detection of any obstacle boundary with workspace penetration depth exceeding is deferred to future work. The computed values for $\lambda_{\max}$ are 0.282 (4-DOF), 0.296 (Shoulder-Joint), 0.209 (SCARA), 0.190 (Universal), 0.101 (PackBot).l
\vspace{-10pt}
\subsection{Number of seed samples}
\SHORT{The number of ray samples ($N_r$ in Algorithm~\ref{alg:seed_generation})  directly affects both coverage and computational cost. Additional rays may intersect boundaries irrelevant to the infeasibility being analyzed, increasing execution time without benefit. Empirically, we find that more than 10 rays leads to a nearly twofold increase in execution time. In the PackBot scenario, we rely solely on boundary tracing intersection points and introduce no additional rays. These observations motivate careful selection of $N_r$ to 
balance coverage and efficiency.}
\vspace{-10pt}
\subsection{Boundary Size}
\label{subsec:boundary_size}
\SHORT{Execution time scales with the size and complexity of the obstacle boundary in configuration space. To study this dependency, we modify environments and 
analyze the resulting performance changes. In the SCARA scenario 
(Fig.~\ref{fig:sim:scara}), increasing the spherical obstacle radius (SCARA 
(small) vs. SCARA (big) in Table~\ref{tab:planning_time}) reduces execution time by approximately two seconds, as the enlarged obstacle shrinks the corresponding free-space connected component and reduces the extent of the boundary to be traced.

Importantly, the relationship between workspace obstacle size and configuration space boundary complexity is not monotonic, as it depends on both obstacle geometry and robot kinematics. For the 4-DOF arm scenarios (link 2 and link 3 in Table~\ref{tab:planning_time}), placing the frame farther from the robot shifts the dominant collision source to the third link, yielding a smaller induced boundary, whereas placing it closer makes the second link the primary source of collision, producing a larger boundary. Beyond obstacle size, spatial arrangement also affects boundary complexity: in the PackBot scenario (Fig.~\ref{fig:sim:packbot}), increasing the table height (PackBot (far)) reduces execution time by simplifying the induced configuration space boundary. Together, these examples illustrate how non-intuitive workspace changes can significantly affect configuration space structure and computational performance.}
\vspace{-5pt}

\subsection{Parallelization}
\SHORT{All computationally intensive components are parallelized on the GPU. Signed 
distance queries are executed via device-level CUDA functions, and forward 
kinematics is computed entirely on-device through a custom kinematic tree, 
eliminating CPU-GPU transfers during core computation. Obstacle and robot 
geometries are modeled using GPU-native primitives, keeping the full collision 
pipeline on-device. During frontier expansion, computed simplicial data is 
transferred to CPU memory after each iteration to free GPU memory for the 
next layer. Since only the current and previous frontier arrays are needed to 
detect revisited simplices, this bounds GPU memory usage independently of the 
total manifold size, enabling operation on consumer-grade hardware.}
\vspace{-7pt}
\subsection{Scalability}
\SHORT{For 6-DOF scenarios, the number of simplices grows exponentially, causing GPU 
memory exhaustion after only 10--15 frontier iterations. Addressing this remains 
an important direction for future work. Promising avenues include configuration 
space decomposition to reduce effective dimensionality \cite{QMP, QRRT}, 
simplified C-space representations \cite{CLIQUE-CONVEX}, more efficient atomic 
data structures \cite{BLEST}, improved hashing mechanisms \cite{HIGHWAYHASH, 
XXHASH}, adaptive triangulation techniques \cite{ADAP_TRI} and seed sampling strategies that initiate tracing from distant regions to accelerate coverage.}
\section{Conclusion} 
\label{sec:conclusion}
This work presents an algorithm for certifying path infeasibility between a given start and goal configuration. The approach traces obstacle-induced separating contact manifolds in the configuration space using a signed distance field representation, enabling direct reasoning about path-connectivity. By reconstructing obstacle boundaries corresponding to the zero level-set, the method determines whether a collision-free path exists between the configurations. The proposed algorithm is inherently parallelizable, allowing efficient exploration of high-dimensional spaces through a frontier-expansion strategy. In addition to detecting infeasibility, the method identifies the specific workspace obstacles responsible for the separation between start and goal configurations. This provides actionable insight into motion planning failure and can guide environment modifications to restore feasibility. The current implementation is limited to 5-dimensional configuration spaces. Scaling to higher-dimensional spaces (6+ DOF) will require algorithmic advances in memory-efficient simplicial reconstruction, which we identify as a key direction for future work, \RR{along with providing a stronger sense of actionability (e.g., finding the minimum number of obstacles to remove) and adaptive methodologies to address the limitations of a fixed global resolution}.
\vspace{-5mm}
\section*{Acknowledgment}
The authors are grateful to the reviewers for their insightful comments and suggestions. The authors would also like to thank, Siargey Kachanovich for his guidance with the manifold tracing implementation.

\bibliographystyle{ieeetr}
\bibliography{references}

\end{document}